\documentclass[letterpaper, 10 pt, conference]{ieeeconf}  % Comment this line out if you need a4paper

\IEEEoverridecommandlockouts                              % This command is only needed if 
\usepackage{graphicx}
\usepackage{subcaption}
\usepackage{algorithm}
\usepackage{algpseudocode}
\usepackage{amsfonts}
\usepackage{amsmath}
\usepackage{amssymb}
\usepackage{hyperref}
\usepackage{color}
\usepackage{cite}
\usepackage{todonotes}
\usepackage{flushend}
\usepackage{multirow}
\usepackage{soul}

\newtheorem{theorem}{Theorem}

\title{\LARGE \bf
EyeLS: Shadow-Guided Instrument Landing System for Intraocular Target Approaching in Robotic Eye Surgery
}

\author{Junjie Yang$^{1*}$, Zhihao Zhao$^{1}$, Siyuan Shen$^{1}$, Daniel Zapp$^{2}$, Mathias Maier$^{2}$, Kai Huang$^{3}$, \\ Nassir Navab$^{1}$, \textit{Fellow, IEEE}, and M. Ali Nasseri$^{2*}$% <-this % stops a space
\thanks{$^{1}$Junjie Yang, Zhihao Zhao, Siyuan Shen and Nassir Navab (Faculty) are with the Department Informatik und Mathematik,
        Technical University of Munich, 80803 Munich, Germany
        {\tt\small \{junjie.yang, zhihao.zhao, siyuan.shen, nassir.navab\}@tum.de}}
\thanks{$^{2}$Mathias Maier and M. Ali Nasseri are with the Klinik und Poliklinik für Augenheilkunde, 
        Klinikum rechts der Isar, 81675 Munich, Germany
        {\tt\small \{daniel.zapp, mathias.maier\}@mri.tum.de}}
\thanks{$^{3}$Kai Huang is with Faculty of the School of Computer Science and Engineering, 
        Sun Yat-Sen University, 510006 Guangzhou, China
        {\tt\small huangk36@mail.sysu.edu.cn}}
\thanks{$^{*}$Junjie Yang is the corresponding author of this paper.}
}

\begin{document}

\maketitle
\thispagestyle{empty}
\pagestyle{empty}

\begin{abstract}
Robotic ophthalmic surgery is an emerging technology to facilitate high-precision interventions such as retina penetration in subretinal injection and removal of floating tissues in retinal detachment depending on the input imaging modalities such as microscopy and intraoperative OCT (iOCT).
% In such minimally-invasive ophthalmic surgeries, the robot is constrained to pivots the surgical instrument (e.g. a needle) around the trocar depending on the input imaging modalities such as microscopy and intraoperative OCT.
Although iOCT is explored to locate the needle tip within its range-limited ROI, it is still difficult to coordinate iOCT’s motion with the needle, especially at the initial target-approaching stage.
Meanwhile, due to 2D perspective projection and thus the loss of depth information, current image-based methods cannot effectively estimate the needle tip’s trajectory towards both retinal and floating targets.
To address this limitation, we propose to use the shadow positions of the target and the instrument tip to estimate their relative depth position and accordingly optimize the instrument tip's insertion trajectory until the tip approaches targets within iOCT's scanning area.
Our method succeeds target approaching on a retina model, and achieves an average depth error of 0.0127 mm and 0.3473 mm for floating and retinal targets respectively in the surgical simulator without damaging the retina.
\end{abstract}

\section{Introduction}
Ophthalmic surgeries demand micron-scale precision for intraocular status estimation and instrument-tissue interaction to avoid anatomic damage.
To guarantee the high-precision instrument manipulation, intraoperative optical coherence tomography (iOCT) is integrated with microscopy in the Operation Room (OP)~\cite{dehghani2023robotic, zhou2023needle, 8793756}, especially the subretinal injection procedure~\cite{Sommersperger:21}.
However, the iOCT-guided surgical applications, such as needle-pose estimation and retina reconstructions~\cite{9502523,keller2018real,DBLP:journals/corr/abs-2301-07204,8324435,Weiss2018Fast5N,yu2015evaluation,Ehlers1306}, take effect only when the surgical instrument is placed within the iOCT's limited scanning range (commonly 6x6 mm$^{2}$ square area).
Although novel methods, such as spectrally encoded reflectometry (SER)~\cite{Tang:22} and 4D OCT are invented to assist the dynamic iOCT-instrument coordination, they still need time-consuming data processing and device validation to achieve a mature surgical integration.
Therefore, it is necessary to investigate more efficient image-guided methods to navigate instruments at the initial target-approaching stage without updating the current surgical hardware in the OP.
\begin{figure}[h]
\centering
    \includegraphics[width=0.98\columnwidth]{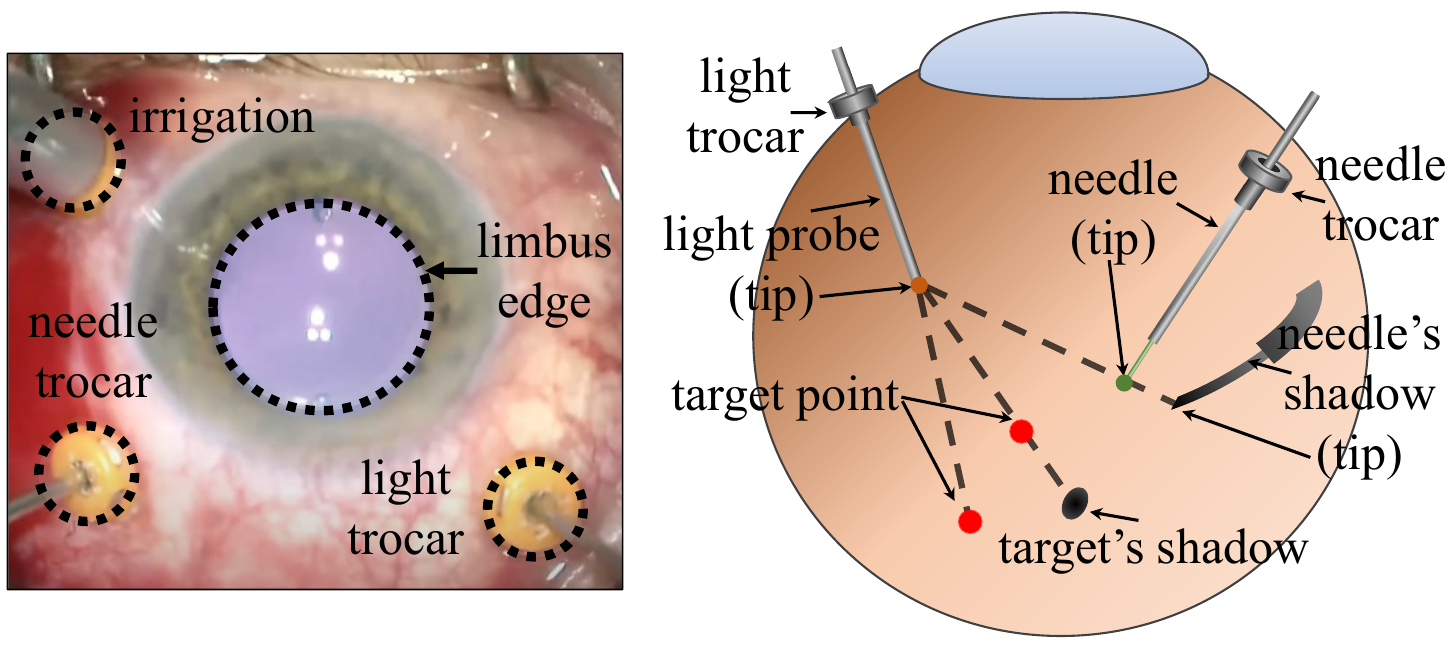}
    \caption{Trocar placement~\cite{youtube} and intraocular components in ophthalmic surgeries.}\label{fig:eyesurgery}
\end{figure}

In microscopy-guided surgeries, the core surgical tasks include extracting the relative instrument-target position and hence optimizing the instrument trajectory.
Some researchers propose to use laser and structured light to cast a visible light marker on the retina for the instrument navigation~\cite{9099104,Trimanual_article,hutchens2012characterization,6290813,doi:10.1177/0278364918778352}.
However, such light-based marker generation can only specify targets on the retina and, therefore, is limited in posterior-segment interventions.
Other methods use microscopy to estimation the absolute position of instruments inside the eye, such as stereo-camera-based depth estimation~\cite{vbaoms} with the difficulty of calibrating the stereo camera system, and deep learning~\cite{9196537,pmlr-v155-kim21a,2301.11839} to implicitly generate motion commands with the  challenge of collecting diverse microscopic image datasets for network training.

As an inspiring work of using shadows, a light-probe control method~\cite{9695979} uses the relative instrument-shadow position to estimate the instrument-retina distance and proposes a method of moving the light probe to eliminate the instrument-shadow overlapping in the microscope image.
However, this method only considers targets on the retina and has the drawback of long-time visual blockage of the retinal target by the instrument, thus causing the difficulty of monitoring the procedure of instrument-tissue interaction and affects the subsequent control of instrument motion.
Also, the instrument's approaching direction towards the retinal target is limited due to its vertical tip trajectory, which is not consistent with the common axial-oriented instrument insertion.

As far as we know, most state-of-the-art methods can only approach retinal targets and hence limits their deployment in the posterior-segment interventions. 
Although few methods can theoretically guide the instrument to reach floating targets by intraocular reconstruction, they are still faced with existing difficulties (i.e., device modification and dataset collection) as discussed above.
Therefore, there lacks a unified theory to use only monocular microscopic images to optimize the instrument tip's trajectory towards both retinal targets in posterior-segment interventions and floating targets in vitreous surgeries.

In this paper, based on the principle of intraocular shadowing, we propose a theorem that the projected intraocular target and it shadow position determine an ideal insertion orientation of the instrument, which consists of an ideal azimuthal orientation and an ideal polar orientation, so that the instrument's axial insertion will pass the target when both instrument-target collision and their shadow collision occur at the same time.
Based on the proposed theorem, we first calibrate the instrument's projected horizontal orientation towards the target and then simultaneously adjust the instrument tip's vertical orientation to ensure the instrument shadow's predicted trajectory will pass the target shadow for both retinal and floating targets during the needle's axial-insertion procedure.

The proposed method has the listed advantages over other state-of-the-art methods:
1) A unified theory to handle both floating and retinal intraocular targets;
2) Explainable and intuitive utilization of shadows for motion control;
3) The side-approaching strategy avoids visual blockage and is consistent with the direction of needle insertion.

\section{Explanation of Concept}
In this paper, we consider the instrument as a needle to cover most surgical cases.
$l$ / $p$ is the fitted line / point of an object in the image, and their subscript has the following meaning: $_{lp}$ / $_{n}$  / $_{ns}$ / $_{t}$ / $_{ts}$ / $nrcm$ refers to light probe / needle / needle-shadow / target / target-shadow / needle's RCM point (trocar), respectively.
The superscript $^{v}$ of a given line indicates that this line $l^{v}$ is vertical and perpendicular to the microscope's imaging plane.
$\mathcal{T}_{n / ns}$ is the tip trajectory of needle and needle's shadow, respectively.
$\sigma_{close}$ is the pixel threshold for choosing horizontal alignment strategy.
$\sigma_{app}$ is the pixel threshold for checking object overlapping.
$\sigma_{align}^{ang / dis}$ is the threshold for checking if horizontal alignment is achieved in the image.
$\mathcal{J}_{n} = \Delta_{R / V / H}$ is the step motion option of R(adius) / V(ertical) / H(orizontal) rotation in each loop of needle trajectory optimization.

In this paper, we emphasize the prior that according to the intraocular shadowing principle, any target $p_{t}$ is always on the segment ended by the light probe $p_{lp}$ and the target's shadow $p_{ts}$ both in the 2D microscope images and in the 3D space.
In ophthalmic surgeries, especially posterior-segment interventions, the patient is under general anesthesia, hence allowing the stable maintenance of the eyeball's pose and visual alignment by the surgeon using eyeball-orbital control~\cite{koyama2023vitreoretinal}.
In this paper, our method is based on the following assumptions: 
1) The surgical cannula needle is modeled as a cylinder with the end-effector and the needle shaft sharing the same axial line; 
2) The instrument's axial rotation ability is not considered in this paper; 
3) The retinal surface is not severely deformed, hence allowing the rectangle approximation of the end-effector's shadow shape;
4) The microscope's perspective and the needle's polar axis are vertically aligned with the eyeball.

Intraocular targets are divided into two types depending on their locations in different ophthalmic surgeries.
In vitreous interventions, objects such as detached tissue are defined as floating targets $\mathbb{P}^{floating}$ in the hollow vitreous space to be removed by peeling or sucking.
Given a light probe tip $p_{lp}$ inside the eye for illumination, the floating target's shadow is cast on the retinal surface, separated from the target itself.
Ideally, the center of the circle-shaped object in the image is regarded as the target point $p_{t} \in \mathbb{P}^{floating}$ and the center of its ellipse-shaped shadow on the retina as the target shadow point $p_{ts}$.
However, in most cases these floating objects are randomly shaped, making it hard to extract target centers as well as their shadow centers. 
Therefore, in such cases we can define the desired collision point between the needle tip and the floating object as the floating target point $p_{t}$ that is normally an edge point of the object in the image.
Subsequently, we use the intersection point of light-target line $l_{lp \to t}$ and the object's shadow edge on the retina as an approximation of the corresponding shadow point $p_{ts}$ according to the shadowing principle.

In posterior-segment interventions, objects, such as the subretinal layers or vessels, are on or beneath the retinal surface with necessary surgical manipulations such as needle insertion and medicine injection.
In such interventions, the needle's preoperatively-defined touching/insertion point on the retina is regarded as a retinal target $p_{t} \in \mathbb{P}^{retinal}$.
Considering the transparency of retinal layers, retinal targets $p_{t}^{retinal}$ are slightly deviated from their shadows $p_{ts}^{retinal}$ regardless of the light probe's location, but they are regarded equal in practical surgeries.

\begin{figure}[h]
    \centering
    \begin{subfigure}{0.492\columnwidth}
        \includegraphics[width=0.99\textwidth]{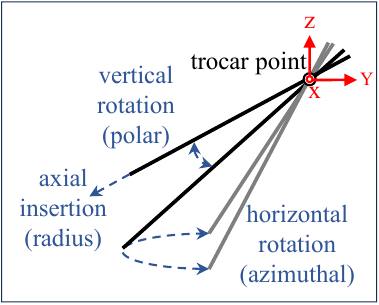}
        \caption{Spherical motion modeling.}
    \end{subfigure}
    \begin{subfigure}{0.492\columnwidth}
        \includegraphics[width=0.99\textwidth]{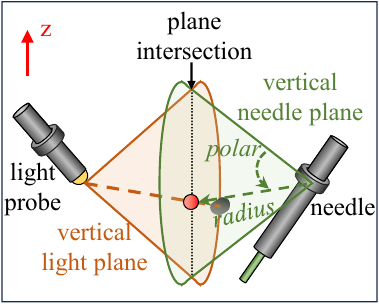}
        \caption{Vertical target exploration.}
    \end{subfigure}
    \caption{Spherical modeling for RCM and target exploration.}\label{fig:intraocularcase}
\end{figure}

Considering the common minimally-invasive surgical setup depicted in Fig.~\ref{fig:eyesurgery}, since trocars are placed through the conjunctiva and sclera to provide tunnels of penetration, the controls of both light probe and needle should keep their motion center around the trocar point, which forms the Remote Center of Motion (RCM) constraints. 
Also, the prerequisite alignment between the microscope and the robot's coordinate system brings about parallelism of the robot's z-axis and the microscope's perspective.
Therefore, the needle motion is modeled by a spherical coordinate system with its origin point at the trocar and three axis (axial, polar and azimuthal) enabled, as shown in Fig.~\ref{fig:intraocularcase}(a).
Therefore, the needle tip's status is represented by a tuple $<r, \theta^{h}, \theta^{v}>$ (radius, horizontal and vertical rotation angles), which is consistent with the spherical coordinate system.

Since the image segmentation technique is already capable of providing high-precision results, the relative position estimation in this paper utilizes the mature industrial tool YOLOv8~\cite{yolov8_ultralytics} to segment intraocular objects.

\section{Method}
\subsection{Major Steps}
The task flow of our proposed trajectory optimization and depth adjustment method is depicted in Fig.~\ref{fig:method_taskflow}, including three major steps.
\begin{figure*}[h]
\centering
    \includegraphics[width=0.98\textwidth]{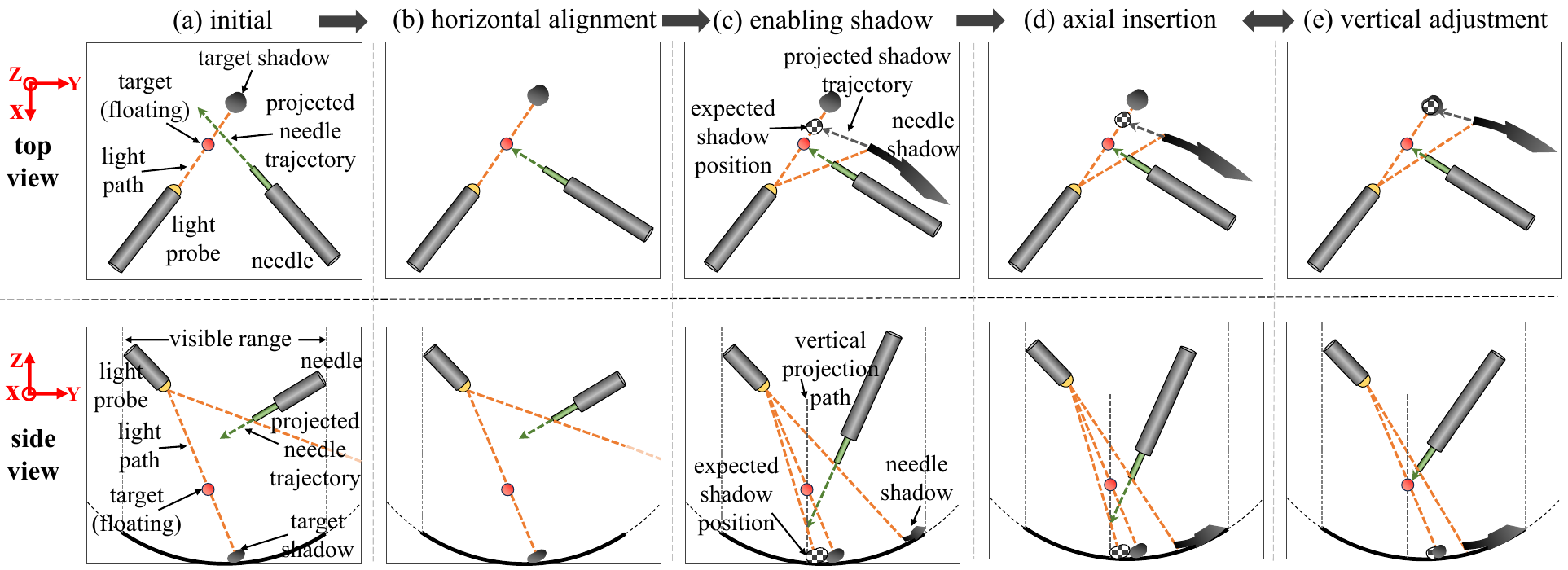}
    \caption{Task flow of the proposed method.}\label{fig:method_taskflow}
\end{figure*}
The first major step \textbf{Horizontal Alignment} is to adjust the needle's horizontal orientation by horizontal rotation until it's heading to the target in the image.
In the next major step \textbf{Shadow Enabling}, the needle tip is vertically moved towards the retina by combining vertical rotation and axial insertion to make its shadow visible in the microscope while the distance between needle-tip and target is maintained within a predefined range $\sigma_{close}$ to prevent the needle tip's unexpected invisibility.
In the third major step \textbf{Shadow Alignment}, the needle continues its \textbf{Axial Insertion} with simultaneous vertical rotation in the polar plane (vertical adjustment) to adjust the needle's predicted shadow trajectory towards the target's shadow.
Finally, the needle's landing procedure stops until both distances $d_{n\to t}$ and $d_{ns\to ts}$ reach the predefined pixel-based threshold $\sigma_{app}$, which guarantees the needle-tip's position inside the iOCT's scanning area for the subsequent micron-precision control.

\subsection{Theory Explanation}
\subsubsection{Horizontal Alignment}
The goal of orientation alignment is to adjust the needle's orientation so that the target is on the needle's insertion trajectory in the projected image, which, however, doesn't guarantee the target passing in the 3D space. 

Considering the fluctuation of the needle's segmented direction due to unstable image segmentation, the orientation alignment is divided into two stages based on the needle-target distance $d_{n\to t}$ compared with threshold $\sigma_{close}$.
When $d_{n\to t} > \sigma_{close}$, the needle is allowed to finish a rough angle-based orientation alignment due to the distance redundancy of safe needle-target interaction.
Having the needle's fixed trocar point $p_{nrcm}$ approximated by the intersection of multiple needle lines, the needle's orientation alignment is then achieved by minimizing the intersection angle between $\vec{v}_{nrcm\to n}$ (also the needle's orientation) and $\vec{v}_{nrcm\to t}$ within the angle threshold $\sigma_{align}^{ang}$.
If $d_{n\to t} \leqslant \sigma_{close}$, the orientation alignment is precisely maintained by limiting the distance from the target $p_{t}$ to the fitted needle line $l_{n}$ within a pixel-based threshold $\sigma_{align}^{dis}$ according to the practical image scale.
The orientation alignment task is running during the whole target-approaching procedure.

\subsubsection{Shadow Enabling}
If there is no needle shadow in the current microscope frame after orientation alignment as Fig.~\ref{fig:case_analysis}(a), the needle should first approach the target by axial insertion until reaching the distance threshold $\sigma_{close}$.
Subsequently, the needle is vertically rotated towards the retinal surface to make its shadow visible in the image.
However, such vertical rotation causes visual decrement of the needle's projected insertion length during the shadow-enabling procedure according to Fig.~\ref{fig:intraocularcase}(b), and thus extra axial insertions are needed to compensate the increment of needle-target distance $d_{n\to t}$ so that both the needle and its shadow appear inside the visible range.

\subsubsection{Shadow Alignment}
\begin{figure}[h]
\centering
    \includegraphics[width=0.85\columnwidth]{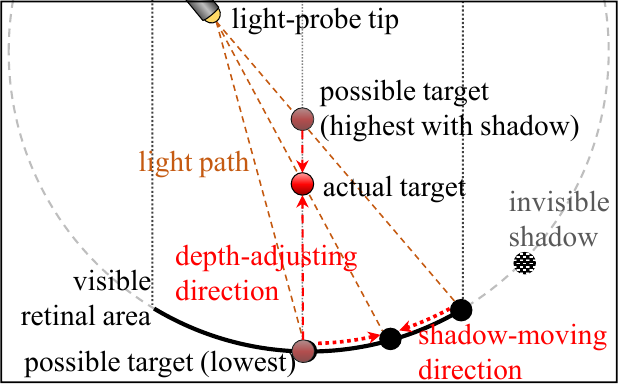}
    \caption{Possible target distribution and their corresponding shadow positions as reference of depth adjustment in the vertical plane containing light-probe tip and target point.}\label{fig:depth_concept}
\end{figure}
With the prerequisite that the light-probe tip is higher than the target, it's determined that the target casts it shadow on the retinal surface, the target located between its shadow and the light-probe tip.
With the light probe statically placed, the shadow of an object indicates the object's relative depth when the object is moving vertically, which acts as the reference of depth adjustment as shown in Fig.~\ref{fig:depth_concept} in the task of overlapping two objects according to \textit{Theorem~\ref{th:approaching}}.

\begin{theorem}\label{th:approaching}
With a static light probe $p_{lp}$, there exists a unique needle-status tuple $<r_{n}, \theta_{n}^{v}>$ after horizontal alignment to make the needle tip $p_{n}$ collide with the target $p_{t}$ if they share both same visual projection and shadow position as Equation~(\ref{eq:th_shadowing}).
\begin{equation}\label{eq:th_shadowing}
\exists! <r_{n}, \theta_{n}^{v}>, p_{n}^{3d} = p_{t}^{3d}\ \text{if}\ p_{n}=p_{t}\ \text{and}\ p_{ns}=p_{ts}
\end{equation}
\end{theorem}
\begin{proof}
In this proof section, $\mathcal{P}$, $l^{v}$ and superscript $\hat{}$ are used to represent a plane, a vertical projection line perpendicular to the image, and an ideal position of an object, respectively.
The microscope's imaging plane is $\mathcal{P}_{img}$.
The needle's orientation is already aligned towards the target as a prerequisite, which leads to the geometrical intraocular status as shown in Fig.~\ref{fig:intraocularcase}(b).

Given a fixed light $p_{lp}$, a target $p_{t} \neq p_{lp}$ and the target's vertical projection line $l_{t}^{v} \perp \mathcal{P}_{img}$, the light probe's vertical shadow-casting sector-shaped plane $\mathcal{P}_{light} = <p_{lp}, r_{lp}, \theta_{lp}^{v}>, \mathcal{P}_{light} \perp \mathcal{P}_{img}$ (the left sector area in Fig.~\ref{fig:intraocularcase}(b)) where the vertical angle $\theta_{lp}^{v}$ defines a unique light ray $l_{lp \to retina}$ cast from the light tip until the retinal surface, and $r_{lp}$ defines a unique point on that light ray.

Similarly, the needle's vertical exploring plane (the right sector area in Fig.~\ref{fig:intraocularcase}(b)) can be defined as $\mathcal{P}_{needle} = <p_{nrcm}, r_{n}, \theta_{n}^{v}>, \mathcal{P}_{needle} \perp \mathcal{P}_{img}$ by the needle trocar $p_{nrcm}$, needle radius $r_{n}$ and its polar angle $\theta_{n}^{v}$ based on the spherical modeling.
Also, the target's vertical projection line $l_{t}^{v}$ is in the needle's exploration plane $\mathcal{P}_{needle}$ after the needle's orientation alignment.
Then, having $\mathcal{P}_{light}$ and $\mathcal{P}_{needle}$, the equation below holds according to the shadowing principle:
\begin{equation}\label{eq:unique_t}
\begin{split}
    l_{t}^{v}\in (\mathcal{P}_{light} \cap \mathcal{P}_{needle}) \\
    p_{t}^{1}, p_{t}^{2} \in l_{t}^{v}, p_{ts}^{1} \neq p_{ts}^{2}\ \text{if }\ p_{t}^{1} \neq p_{t}^{2}
\end{split}
\end{equation}
Therefore, 
\begin{equation}\label{eq:unique_ts}
\begin{split}
p_{n} = \hat{p}_{n} = p_{t}\ \text{if}\ p_{ns} = p_{ts} \\
\hat{p}_{n} = l_{lp\to t} \cap l_{nrcm\to t} \\
\hat{p}_{n} \Leftrightarrow <\hat{r}_{n}, \hat{\theta}_{n}^{v}>\ \text{if}\ \hat{p}_{n} \in \mathcal{P}_{needle}
\end{split}
\end{equation}
if the needle $p_{n}$ is vertically moving on the target's line $l_{t}^{v}$.
Since $\hat{p}_{n} \in \mathcal{P}_{needle}$, the current tip position $p_{n}$ is corresponding to an ideal status tuple $<\hat{r}_{n}, \hat{\theta}_{n}^{v}>$ with a fixed trocar point $p_{nrcm}$.
As a result, there exits a unique needle status to overlap needle $p_{n}$ and target $p_{ts}$ in the 3D space with their shadows $p_{ns}$ and $p_{ts}$ also overlapped.

Due to the fact that the needle's trocar ($p_{nrcm}$) is fixed on the eye surface, a sequence of needle's axial insertion length (radius) is mapped to a unique needle-tip trajectory $\mathcal{T}_{n}$ and its shadow-tip trajectory $\mathcal{T}_{ns}$ if the light probe is kept fixed.
Therefore, there is a unique polar-azimuthal orientation tuple for the needle to ensure that the 3D target point is passed by the needle-tip trajectory.
According to the shadowing principle, the needle's shadow tip will be cast at the target's shadow point when the needle tip is collided with the target.
Therefore, there is a unique polar-azimuthal angle tuple for the needle to generate its ideal needle-tip trajectory $\hat{\mathcal{T}}_{n}$ that passes target point $p_{t}$ and shadow-tip trajectory $\hat{\mathcal{T}}_{ns}$ that passes target shadow $p_{ts}$ by the needle's insertion (axial translational movement).
\end{proof}

According to the \textit{Theorem~\ref{th:approaching}}, we inversely use the position between needle's shadow trajectory $\mathcal{T}_{ns}$ and target shadow $p_{ts}$ as the reference of vertical-orientation adjustment.
In practice, the future shadow trajectory is predicted by extending the current shadow along its current orientation as $\mathcal{T}_{ns} \approx \mathcal{T}_{ns}^{pred} \approx l_{ns}$, and hence simultaneous shadow alignment is achieved during the axial insertion to adjust the relative position between the target shadow's location $p_{ts}$ and the needle shadow's predicted trajectory $l_{ns}$.
Since the needle is already horizontally aligned towards the target, the needle tip is determined to reach the target in the projected image.
Therefore, we use the projected light-target ray $l_{lp\to t}$  to calculate the needle's expected shadow position $p_{esp}$ by calculating the intersection point of light ray $l_{lp\to t}$ and the predicted shadow trajectory $l_{ns}$.
This expected shadow position $p_{esp}$ indicates the prediction of the needle's shadow-tip position when the needle tip is visually overlapped with the target in the image with all points $p_{lp}$, $p_{t}$, $p_{ts}$, $p_{n}$ and $p_{ns}$ inside the same vertical exploration plane as shown in Fig.~\ref{fig:depth_concept}.
Then, the simultaneous shadow alignment is achieved by overlapping $p_{ts}$ and $p_{esp}$ which is covered by Fig.~\ref{fig:case_analysis}(b)-(f) until the needle shadow $p_{ns}$ reaches the ideal shadow position $p_{esp} \approx p_{ts}$.
Besides, the horizontal alignment is loosen to be tolerant with non-strict azimuthal rotations since the subsequent vertical offset can be compensated by the vertical adjustment guided by the shadow relationship between needle and target. 

\subsection{Decision Making}
\begin{figure}[ht]
    \centering
    \begin{subfigure}{0.48\columnwidth}
        \includegraphics[width=0.95\textwidth]{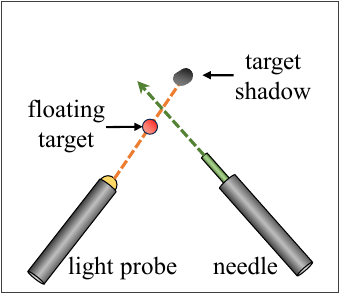}
        \caption{no visible $p_{ns}$}
    \end{subfigure}
    \begin{subfigure}{0.48\columnwidth}
        \includegraphics[width=0.95\textwidth]{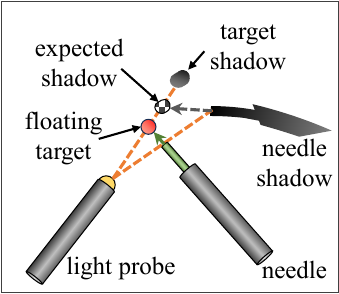}
        \caption{$\mathcal{T}_{n}^{3d}$ lower than $p_{t}^{3d}$}
    \end{subfigure}
    \begin{subfigure}{0.48\columnwidth}
        \includegraphics[width=0.95\textwidth]{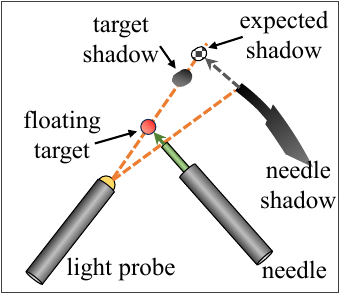}
        \caption{$\mathcal{T}_{n}^{3d}$ higher than $p_{t}^{3d}$}
    \end{subfigure}
    \begin{subfigure}{0.48\columnwidth}
        \includegraphics[width=0.95\textwidth]{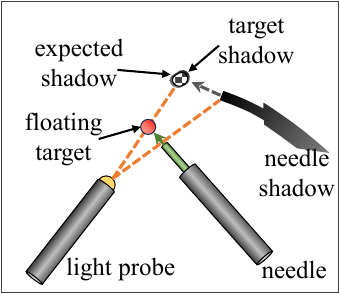}
        \caption{aligned}
    \end{subfigure}
    \begin{subfigure}{0.48\columnwidth}
        \includegraphics[width=0.95\textwidth]{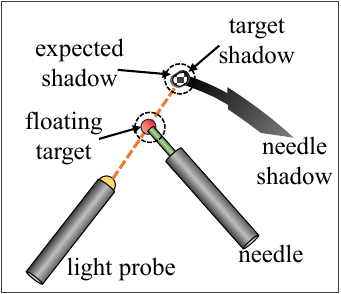}
        \caption{reach floating target}
    \end{subfigure}
    \begin{subfigure}{0.48\columnwidth}
        \includegraphics[width=0.95\textwidth]{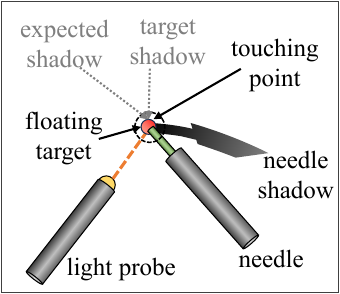}
        \caption{reach retinal target}
    \end{subfigure}
    \caption{Exclusive situations for decision making in projected microscope images.}\label{fig:case_analysis}
\end{figure}
Given a microscope image, the needle navigation problem is classified as one of six situations in Fig.~\ref{fig:case_analysis} with corresponding motion commands generated as \textbf{Algorithm~\ref{algo:shadow_align}}.
The horizontal alignment and its real-time maintenance are basic assertions as \textbf{Algo-1} line 1.
After the recognition of point $p_{lp}$, $p_{ts}$ and $p_{esp}$, the distance from $p_{lp}$ to the latter two point can be calculated as $d_{lp\to ts}$ and $d_{lp\to esp}$.
\begin{algorithm}[h]
\caption{Procedure of Shadow Alignment}\label{algo:shadow_align}
\begin{algorithmic}[1]
    \State AssertOrientationAligned()
    \State \textbf{if} $|d_{lp\to esp}-d_{lp\to ts}| \leqslant \sigma_{app}$
    \State \ \ \ \ \textbf{if} ($\|p_{n}-p_{t}\|\leqslant \sigma_{app}$) \textbf{and} ($\|p_{ns}-p_{ts}\|\leqslant \sigma_{app}$)
        % shadow aligned
    \State \ \ \ \ \ \ \ \ StartiOCT()
    \State \ \ \ \ \textbf{else}
    \State \ \ \ \ \ \ \ \ \textbf{if} $\|p_{n}-p_{ns}\|\leqslant \sigma_{app}$
    \State \ \ \ \ \ \ \ \ \ \ \ \ $\mathcal{J}_{n} = \Delta_{R}^{out}\ \& \ \Delta_{V}^{up}$
    \State \ \ \ \ \ \ \ \ \textbf{else}
    \State \ \ \ \ \ \ \ \ \ \ \ \ \textbf{if} $d_{nrcm\to n} < d_{nrcm\to t}$
    \State \ \ \ \ \ \ \ \ \ \ \ \ \ \ \ \ $\mathcal{J}_{n} = \Delta_{R}^{in}$
    \State \ \ \ \ \ \ \ \ \ \ \ \ \textbf{else}
    \State \ \ \ \ \ \ \ \ \ \ \ \ \ \ \ \ $\mathcal{J}_{n} = \Delta_{R}^{out}$
    \State \textbf{else}
    \State \ \ \ \ \textbf{if} $d_{lp\to ts} < d_{lp\to esp}$
    \State \ \ \ \ \ \ \ \ $\mathcal{J}_{n} = \Delta_{V}^{down}$
    \State \ \ \ \ \textbf{else}
    \State \ \ \ \ \ \ \ \ $\mathcal{J}_{n} = \Delta_{V}^{up}$
    \State MoveNeedle($\mathcal{J}_{n}$, $\Delta_{R}$)
\end{algorithmic}
\end{algorithm}

If $d_{lp\to ts}$ and $d_{lp\to esp}$ are approximately equal within the threshold $\sigma_{app}$, both needle and its shadow are considered to be correctly aligned towards their target position and the subsequent high-precision iOCT navigation is triggered in the surgery as line 2-4 
Otherwise, the shadow-aligned needle should its axial insertion by step length $\Delta_{R}$ depending on if it's already beyond the target or falling behind as \textbf{Algo-1} line 9-12 and depicted in  Fig.~\ref{fig:case_analysis}(d).

If the estimated shadow position $p_{esp}$ is not overlapped with the target shadow $p_{ts}$, the relationship between $d_{lp\to ts}$ and $d_{lp\to esp}$ defines the relative position between the needle tip and the target.
If $d_{lp\to ts} < d_{lp\to esp}$, the needle-tip's insertion trajectory is above the target's 3D position, which causes the needle shadow has a further visual location from the light tip than the target shadow on the retinal area that the light probe is pointing at as shown in Fig.~\ref{fig:case_analysis}(c).
Correspondingly, the needle should be vertically rotated downwards to calibrate the shadow's trajectory through the target shadow as line 14-15.
By contrast, if $d_{lp\to ts} > d_{lp\to esp}$, the needle-tip's insertion trajectory is below the target's 3D position, generating a needle shadow trajectory closer to the light tip as Fig.~\ref{fig:case_analysis}(b).
Therefore, the needle should be vertically rotated upwards so that the shadow's trajectory can go through the target shadow as line 16-17.

It is worth noticing that in some situations the needle approaches the retinal surface too fast before reaching the target.
In order to eliminate the risk of retinal damage, a safety-check step with the highest priority of execution is needed as line 6-7 to move the needle tip away from the retina by rotating upwards and axial fallback to ensure the distance between needle $p_{n}$ and its shadow $p_{ns}$ is bigger than threshold $\sigma_{app}$.

\section{EXPERIMENT}
The experiments are divided into two categories: simulation and model test. 
For simulation, the automatic motion test is conducted in a Unity-based simulator with ground-truth positions and orientations given to prove the correctness of our decision making.
For the model test, a retinal model setup is 3d-printed to check whether the algorithm can guide the needle towards the target with necessary segmentation metrics provided.

\subsection{Simulation}
\begin{figure}[h]
    \centering
    \includegraphics[width=0.9\columnwidth]{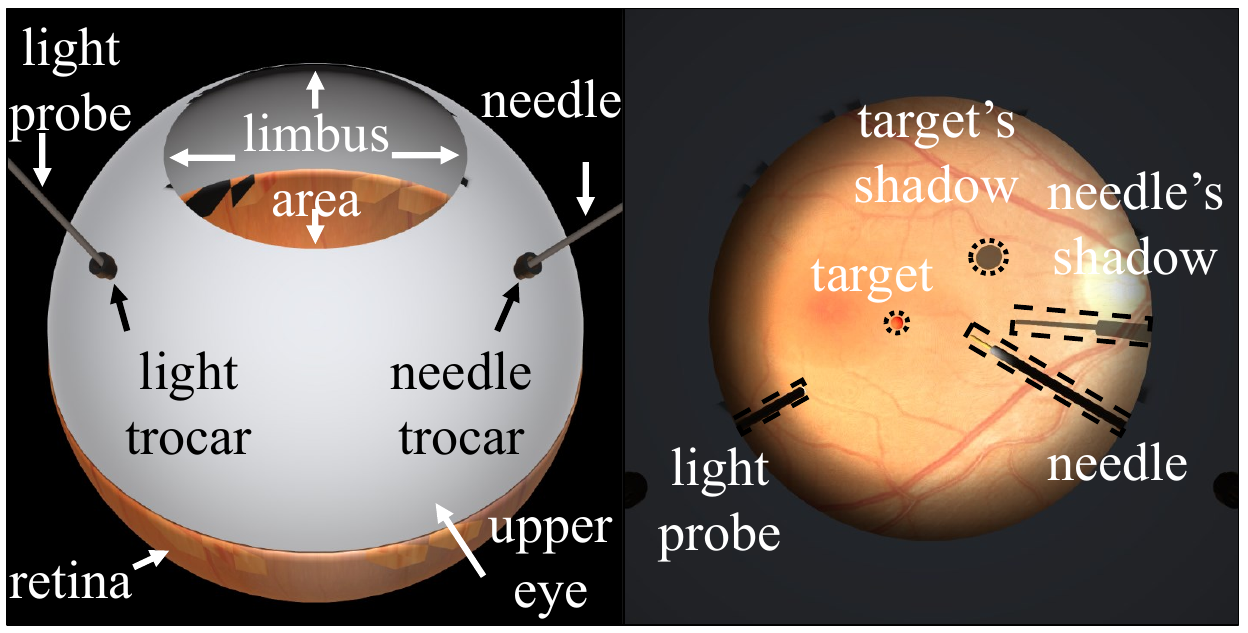}
    \caption{The simulator used for our experiment.}\label{fig:simulator}
\end{figure}
\subsubsection{Setup}
A Unity-based simulator as shown in Fig.~\ref{fig:simulator} is developed by our lab and used for the experiment of needle navigation towards retinal targets.
To simplify the complicated light control task, the light probe is statically positioned to ensure adequate brightness and shadow shaping.
The radius of the simulated spherical eye is 12 mm with a 6-mm-radius limbus area.

In this experiment, the needle-target distance is presented to demonstrate how close two objects are to each other and thus the performance of our target approaching algorithm.
Meanwhile, the needle-retina distance is presented to show if the needle tip is too close to the retina, which may cause unexpected tissue damage before conducting surgical manipulations.

Here are the hyper parameters used in the simulation: threshold for angle-based alignment $\sigma_{align}^{ang}$=3$^{\circ}$, threshold for distance-based alignment $\sigma_{align}^{dis}$=2 pixels, threshold to start distance-based alignment $\sigma_{close}$=100 pixels, threshold to trigger iOCT navigation $\sigma_{app}$=15 pixels, step angle for rotation $\Delta_{V/H}$=0.2$^{\circ}$ and step length for zooming $\Delta_{R}$=0.167 mm.
In some experiment cases, the motion parameters need further tuning to avoid being stuck at certain status.

\subsubsection{Metric}
\begin{table}[h]
\centering
\caption{Information of Target Types in Simulation}\label{tb:targetinfo}
\begin{tabular}{|c|c|c|c|c|}
\hline
type      & percent    &   * dis-x       &  * dis-y       &  * dis-z        \\ \hline
floating  & 46.78\%    &    0.60 (3.00)  &  -1.00 (4.71)  &  -9.02 (1.98)   \\ \hline
retinal   & 53.22\%    &   -0.02 (5.03)  &   0.06 (7.08)  &  -11.48 (0.08)  \\ \hline
\multicolumn{5}{|l|}{* Axis distribution (mm) consists of mean (variance).}  \\ \hline
\end{tabular}
\end{table}
For the validation of our method in the simulator, we generate \textbf{1926} random intraocular targets with their mean (variance) along each axis in TABLE~\ref{tb:targetinfo}, all targets within the limbus-visible area.
All simulated microscope images are shaped as 1024x1024 pixels.

\begin{figure}[h]
    \centering
    \includegraphics[width=0.88\columnwidth]{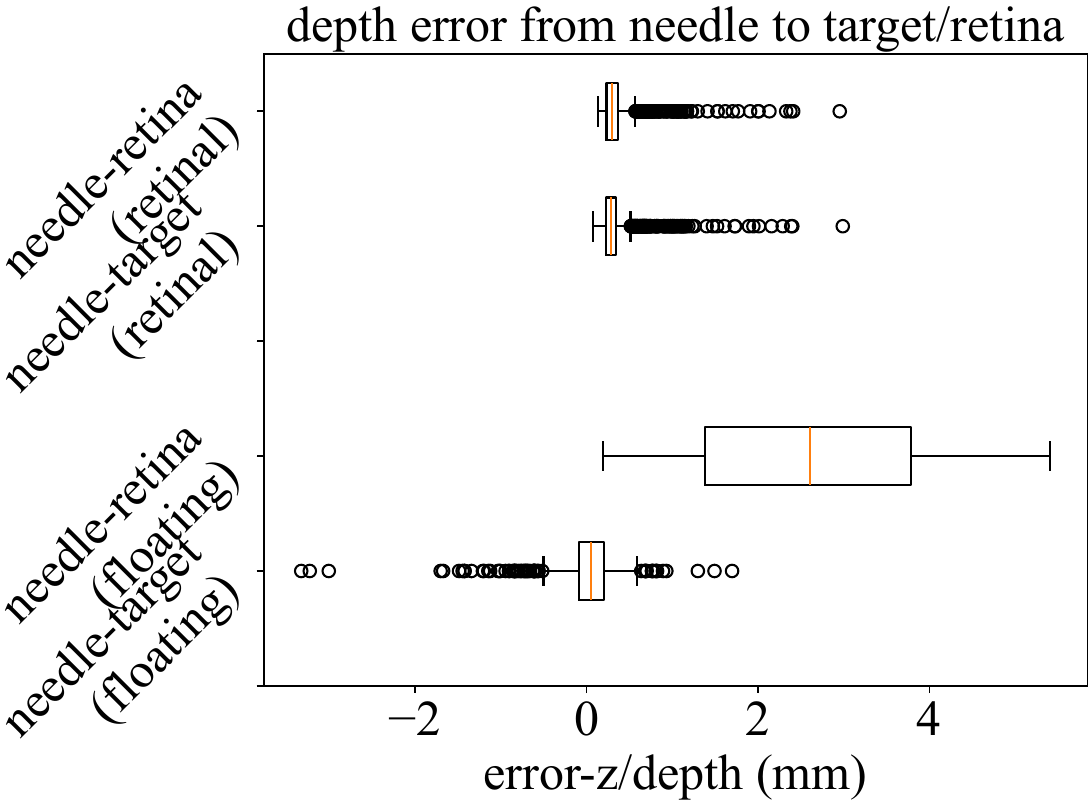}
    \caption{Needle-target depth errors for two target types.}\label{fig:deptherror}
\end{figure}
With only 1 failure as shown in Fig.~\ref{fig:diffcase} and another 3 stuck cases with tunable threshold variables, we collect the relative depths from the needle-tip to the target and from needle's shadow to the target' shadow in the remaining 1922 successful cases to demonstrate the result of vertical alignment.
The final depth error of needle to both target types after approaching is collected together with the vertical distance from the needle to the retinal surface in Fig.~\ref{fig:deptherror} and their mean values presented in Table~\ref{tb:zerror_table}.
\begin{table}[h]
\centering
\caption{Mean of Needle Tip's Depth(z-axis) Error (mm)}\label{tb:zerror_table}
\begin{tabular}{|c|c|c|c|c|}
\hline
target type  &  \multicolumn{2}{|c|}{floating} &  \multicolumn{2}{|c|}{retinal}  \\ \hline
reference   & target & retina & target & retina \\ \hline
mean &  0.0127 & 2.6161   &  0.3473  &  0.3628 \\ \hline
\end{tabular}
\end{table}
Fig.~\ref{fig:deptherror} shows the overall target-approaching performance with floating and retinal targets discussed separately.
As for floating targets, although the needle-tip position is distributed around the target with both higher and lower situations with its average depth error around 0 mm, the needle-retina depth error is always larger than 0 mm and within 4 mm, marking that the needle tip is always above the retinal surface without unexpected touching.
Similarly, the needle never touches the retinal surface when approaching retinal targets.
Also, since retinal targets are located on the retinal surface, the depth-distributions of needle-target and needle-retina are consistent.
Therefore, this depth-error distribution shows the safe needle placement of our method without damaging the retinal tissue.
At the same time, according to the detailed error values in Table.~\ref{zerror_table}, the mean value of needle-target distance along z-axis in  shows that the final needle-target distance is smaller than 0.5 mm, which proves the feasibility of our image-based target-approaching method in ophthalmic surgeries.
Also, the mean value of needle-retina distance along z-axis larger than 0.1 mm proves the safety of our algorithm by the avoidance of unexpected needle-retina collision which prevents unnecessary tissue damage.
Thanks to the inspiration of using needle-shadow tip collision to predict retina touching in \cite{9695979}, we achieved damage-free needle-tip control in our method.

\begin{figure*}[h]
\centering
    \begin{subfigure}{0.245\textwidth}
        \includegraphics[width=0.99\textwidth]{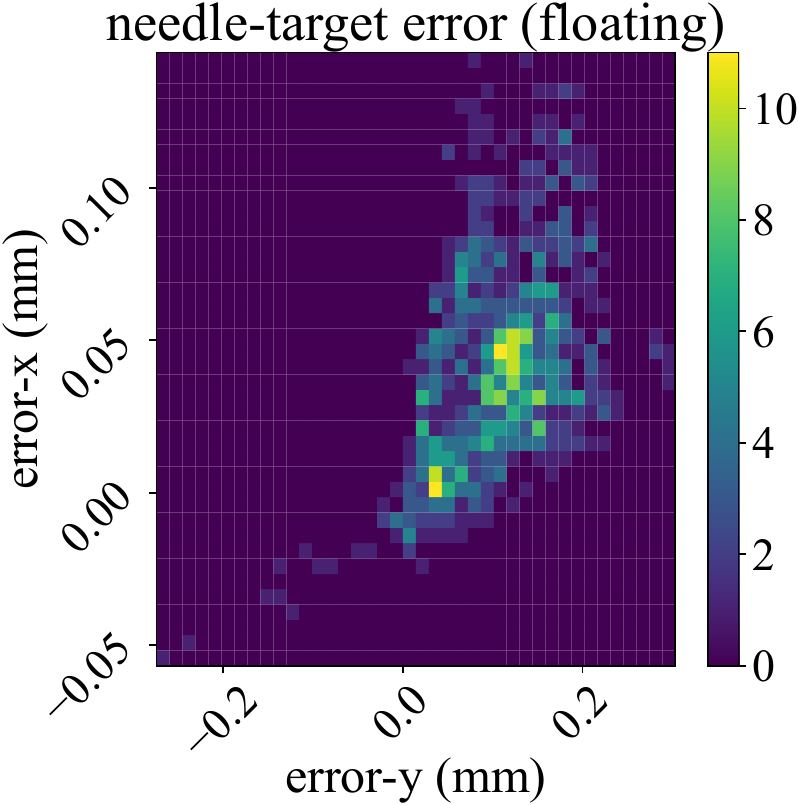}
        \caption{needle error (floating)}
    \end{subfigure}
    \begin{subfigure}{0.245\textwidth}
        \includegraphics[width=0.99\textwidth]{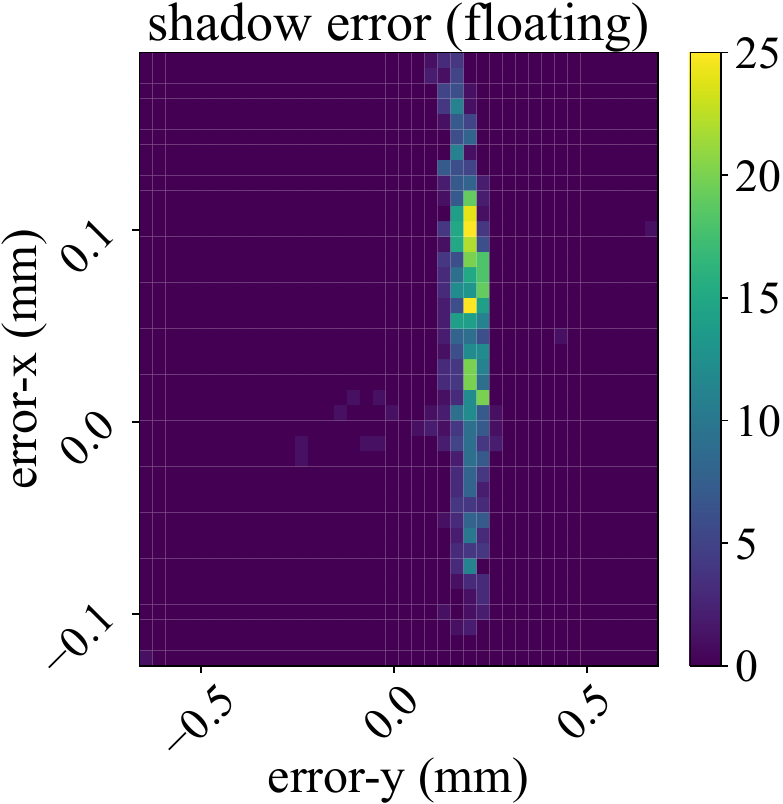}
        \caption{shadow error (floating)}
    \end{subfigure}
    \begin{subfigure}{0.245\textwidth}
        \includegraphics[width=0.99\textwidth]{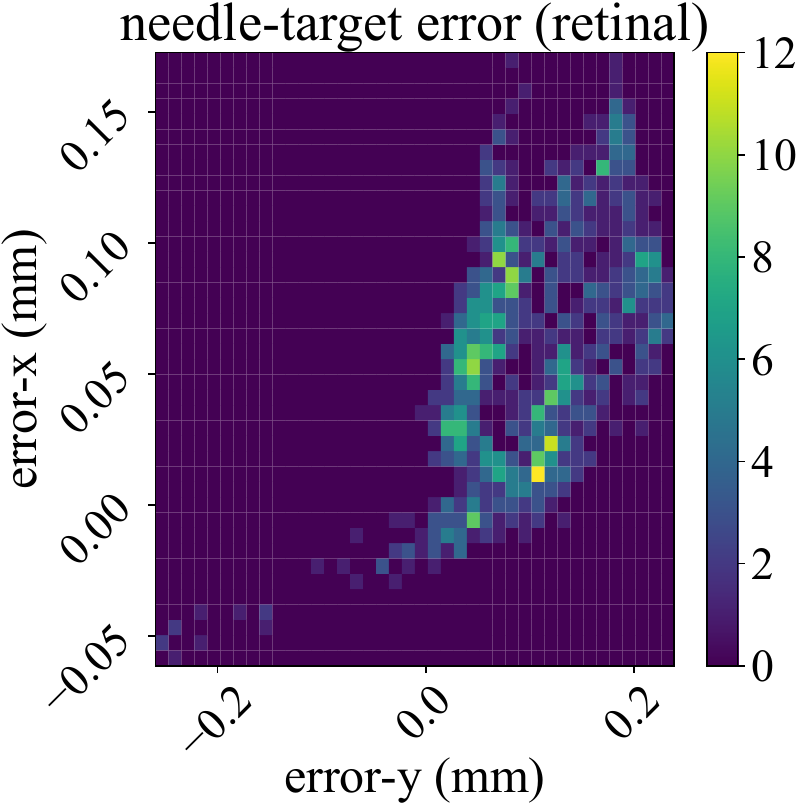}
        \caption{needle error (retinal)}
    \end{subfigure}
    \begin{subfigure}{0.245\textwidth}
        \includegraphics[width=0.99\textwidth]{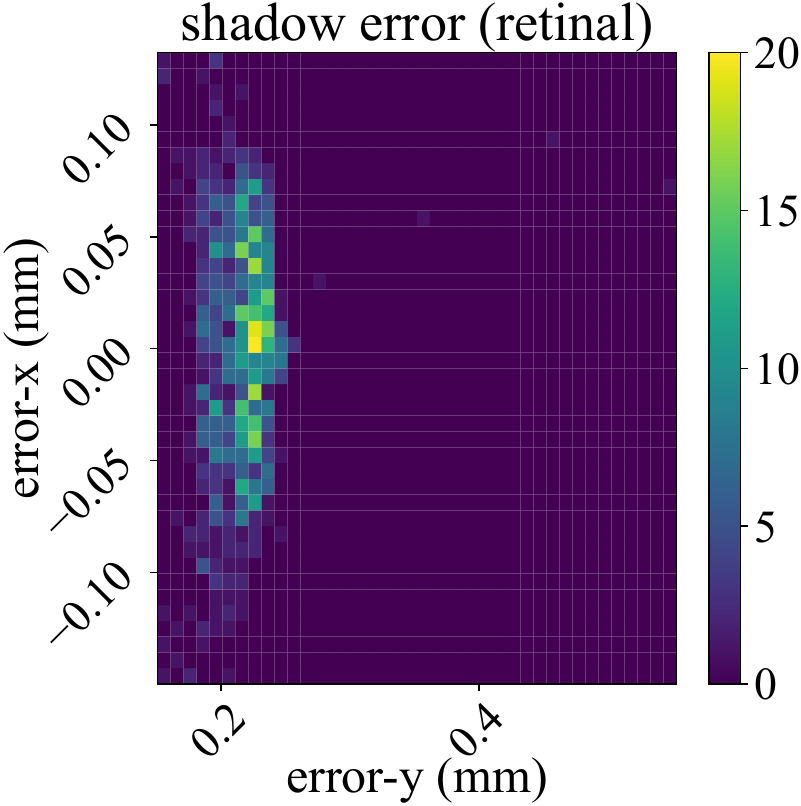}
        \caption{shadow error (retinal)}
    \end{subfigure}
    \caption{The needle and shadow's error to their target position in x-y plane after approaching.}\label{fig:xyerror}
\end{figure*}
The needle-target and shadow errors in the x-y plane in Fig.~\ref{fig:xyerror} demonstrate the overlapping of needle tip and targets of two types.
According to the error distribution, the needle tip is located within the 1mm-radius surrounding area of the target. 
This overlapping ensures that the needle tip has reached the target projected position, which is consistent with the vertical-overlapping theory in this paper and \cite{9695979}.
Furthermore, Since the common scanning diameter of iOCT is 6 mm, the approaching result shown in Fig.~\ref{fig:xyerror} proves that the approaching method in this paper can properly place the needle tip inside the iOCT's scanning range for later micron-level manipulation.

Combining the error distribution along z-axis (depth) in Fig.~\ref{fig:deptherror} and Table.~\ref{zerror_table} with the error distribution of x-y plane in Fig.~\ref{fig:xyerror}, the 3D error distribution proves that our proposed target-approaching method enables the needle tip to approach floating and retinal targets within 1 mm and allows the subsequent needle manipulations such as insertion and sucking without damaging intraocular tissues.

\begin{figure}[h]
\centering
    \begin{subfigure}{0.9\columnwidth}
        \includegraphics[width=0.95\textwidth]{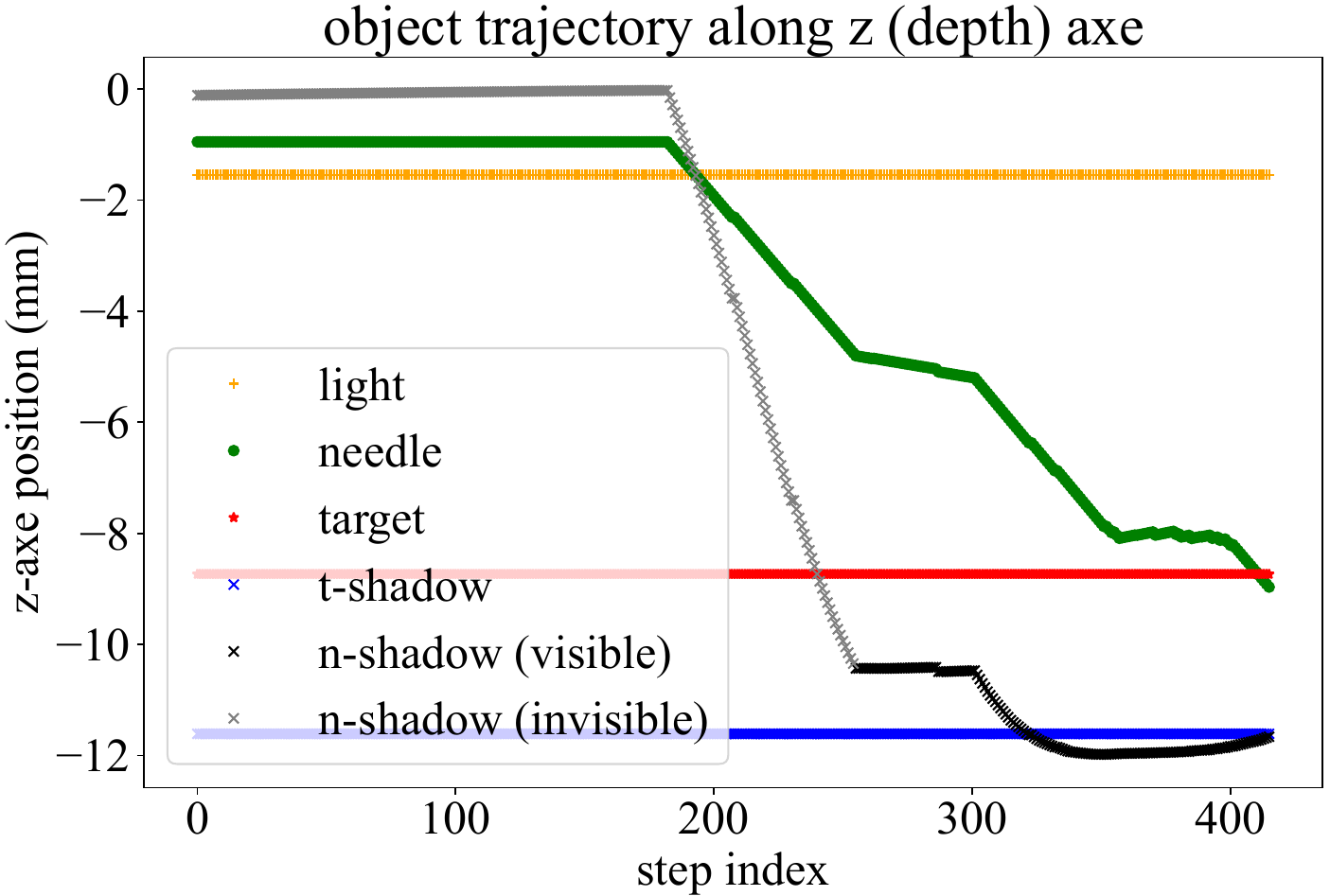}
        \caption{z-axis (floating)}
    \end{subfigure}
    \begin{subfigure}{0.9\columnwidth}
        \includegraphics[width=0.95\textwidth]{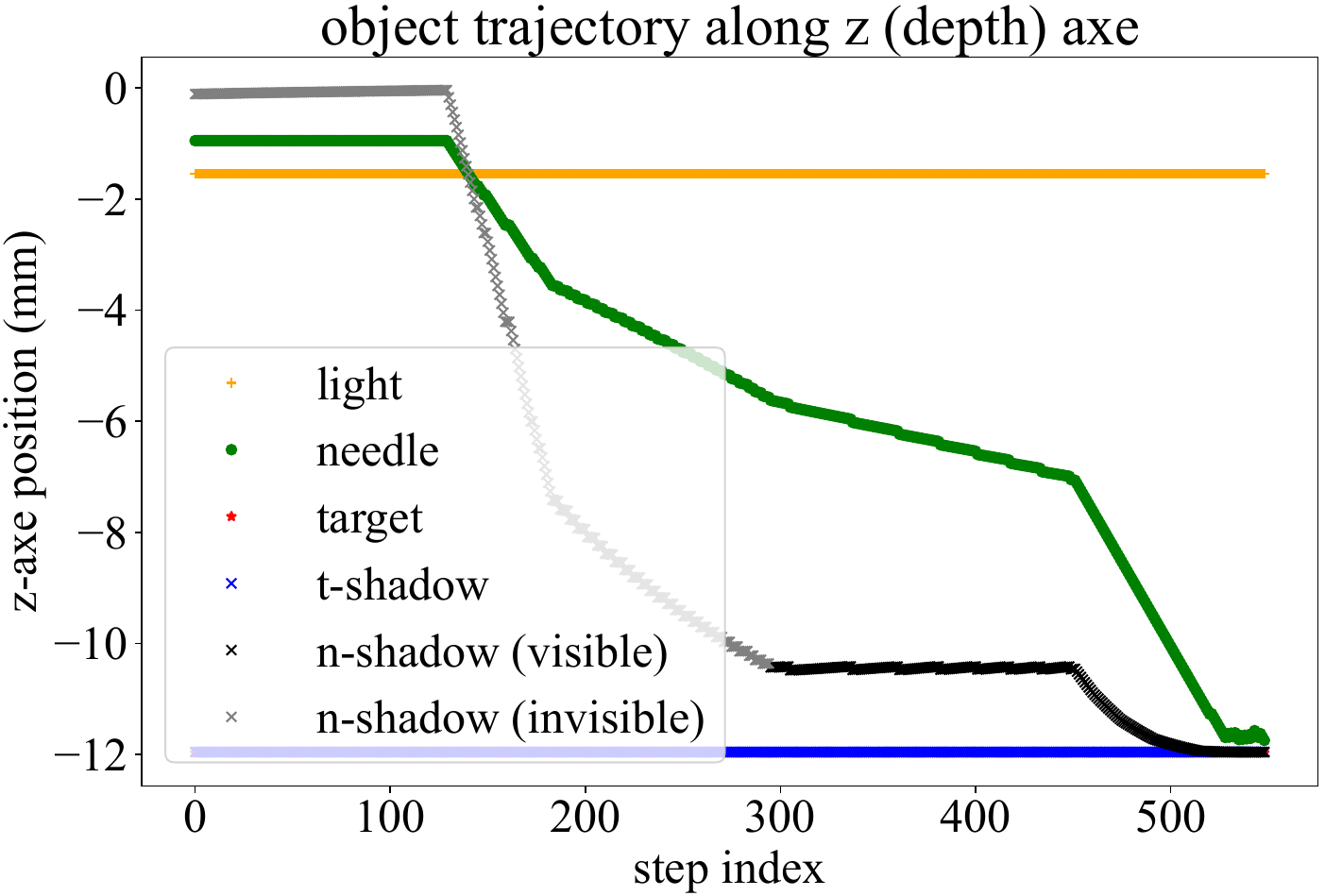}
        \caption{z-axis (retinal)}
    \end{subfigure}
    \caption{The simulated needle-tip's trajectory along z axis (depth) towards the floating and retinal targets, respectively with a given static light. }\label{fig:exp_simu_trajec}
\end{figure}
Fig.~\ref{fig:exp_simu_trajec} shows the procedure of approaching two types of targets along the z axis (depth) to help better understand how the needle tip is getting close to the target.
As for the floating target in Fig.~\ref{fig:exp_simu_trajec}(a), the needle first adjusts its orientation towards the target by horizontal rotation between step-index $[0,200]$, which doesn't change the needle tip's depth according to the spherical modeling, and then optimize the needle-tip's depth to approach the target between step-index $[200,300]$.
Simultaneously, coordinated axial insertion and vertical rotation (upwards) helps maintain the needle's depth but adjust the shadow's position until both needle tip and it shadow tip are overlapped with those of the target.

As for the retinal target in Fig.~\ref{fig:exp_simu_trajec}(b), the z (depth) trajectory also follows the predefined step division with the target overlapped with its shadow on the retina of a 12mm-radius eye.
Similar horizontal alignment and depth optimization are also achieved in the retinal-target case.
However, according to the depth trajectory, a unique depth fluctuation of the needle's shadow is observed between step-index $[300, 460]$, which is caused by the emergent retinal collision avoidance to ensure that the needle tip will not damage the retina before surgical tasks.
Afterwards, the needle keeps a smooth axial insertion towards the target to finally approach the target.
Since floating targets are always above the retinal surface and hence separated with their shadows, such emergent collision avoidance doesn't exist during the approaching procedure in the floating scenario.
\begin{figure}[h]
    \centering
    \includegraphics[width=0.8\columnwidth]{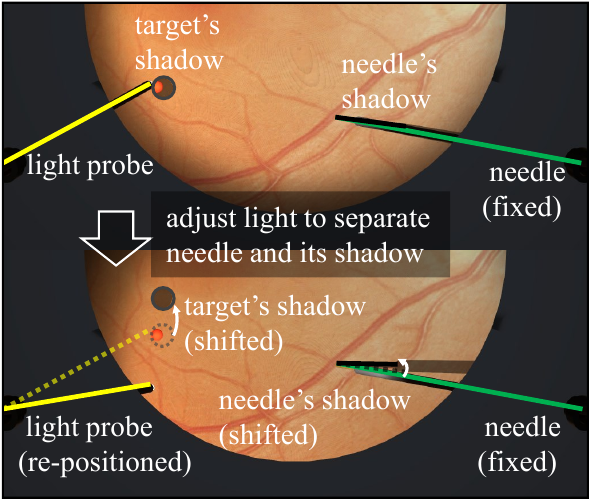}
    \caption{The difficult case that a manual light adjustment is needed to separate the needle from it's shadow to enable further target approaching, which applies to \cite{9695979}.}\label{fig:diffcase}
\end{figure}

Then, the failure case is presented in Fig.~\ref{fig:diffcase} with a potential solution.
As a typical challenge for shadow-assisted needle navigation, the overlapping of the needle and its shadow prevents the extraction of the needle shadow's tip position and orientation and thus the prediction of the needle's future trajectory.
As shown in Fig.~\ref{fig:diffcase}, the light tip is incorrectly fixed at the current position to make $p_{lp}$, $p_{n}$ and $p_{ns}$ almost on the same line, causing the failed calculation of $p_{esp}$. 
According to the light-probe control method in \cite{9695979}, a potential solution to this challenge is the automatic execution of online light-tip movement to separate the needle shadow's position away from the needle in the visible area according to the intraocular shadowing principle by rotating the light probe clockwise until the shadow appears with a closer position to the retinal center than the needle.
Meanwhile, the target's shadow is also deviated from the original position, which enlarges the distance between the needle and its shadow's predicted trajectories.
Following the dicussion of light-probe adjustment, we manually rotate the light probe in the simulator to separate the needle with its shadow, and consequently succeed the target approaching task.

\subsection{Model Test}
\begin{figure}[h]
\centering
    \includegraphics[width=0.98\columnwidth]{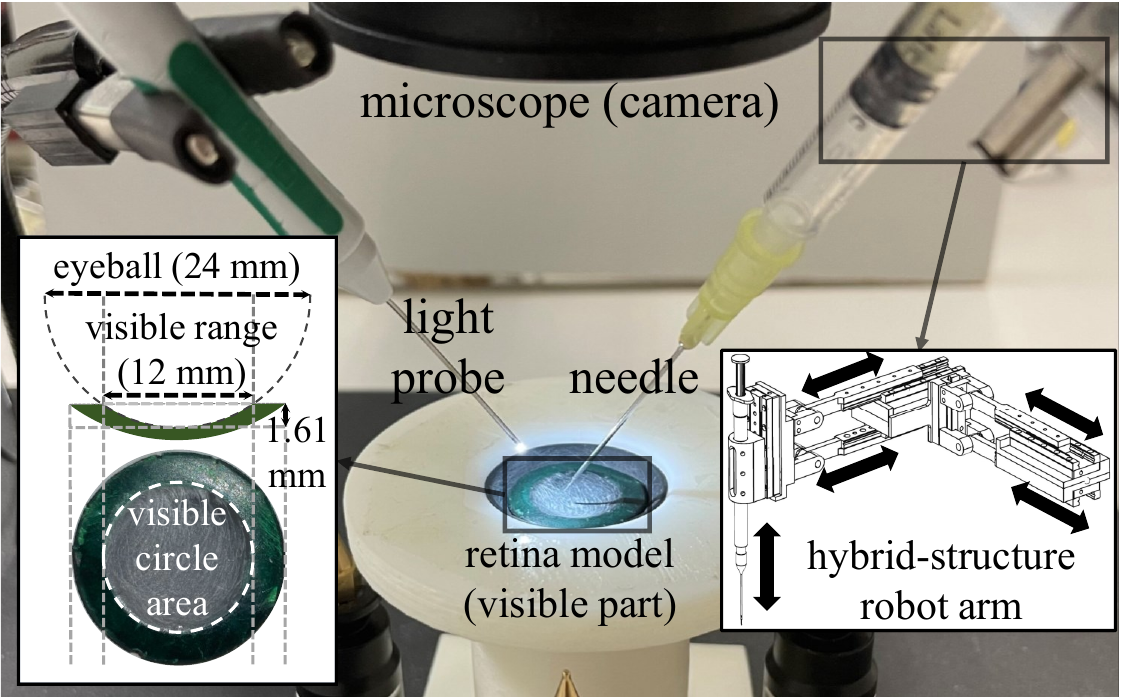}
    \caption{The hardware setup for experiments.}\label{fig:real_setup}
\end{figure}
\subsubsection{Setup}
Target-Approaching experiments are also conducted on a retina model to provide a qualitative effectiveness proof of the shadow-alignment algorithm in this paper.
The hardware setup is shown in Fig.~\ref{fig:real_setup} which contains a light probe, a retina model, a microscope and a needle mounted on an hybrid-structure robot~\cite{8818660}.

In this paper, we use a Geuder cannula (model G-34285, 23 gauge) for image segmentation and the subsequent target approaching.
An RCM-compatible actuator is used to listen to the image-based controller for motion commands and actuate certain motors.
For the simplicity of implementation, a loose robot initialization is achieved to roughly place the robot under the microscope without strict hand-eye calibration, which is compensated by the offset-based control strategy.
The needle's shadow has already been activated in the visible area.
We define the following runtime hyper-parameters: threshold to start distance-based alignment $\sigma_{close}$=100 pixels, threshold to trigger iOCT navigation $\sigma_{app}$=45 pixels, step angle for rotation $\Delta_{V/H}$=0.05$^{\circ}$ and step length for insertion $\Delta_{R}$=0.01 mm.

As for the recognition of image components, the light probe and the target are static objects that can be manually marked in the image.
Also, two types of targets are abstracted by their center points and manually set in the original microscopy image, target overlapped with its shadow at $[512, 512]$ for the retinal case, and separated as $[512,512]$ and $[671,391]$ for the floating case in the image.
Meanwhile, since no ground-truth positions of dynamic objects are given during the experiment, the needle and its shadow positions are segmented in each frame as the input of decision making.
Therefore, we train an \textbf{l-seg} segmentation model implemented by YOLOv8~\cite{yolov8_ultralytics} with 158 images in total, 126 images for training and 32 images for validation.
The segmentation model's final metrics after training are public in TABLE.~\ref{tb:real_metric} running on a hardware combination of a 6-core CPU and an RTX-A5000 GPU.
\begin{table}[h]
\centering
\caption{Training Metrics}\label{tb:real_metric}
\begin{tabular}{|c|c|c|c|}
\hline
Precision &  Recall  &  mAP50  &  mAP50-95   \\ \hline
0.981     &  0.985   &  0.984  &  0.837      \\ \hline
\multicolumn{4}{|c|}{543 epochs in total with batch size 8.}  \\ \hline
\end{tabular}
\end{table}
After segmentation to obtain contours of instruments, we model the instrument's orientation by the fitted center line of its contour and use the intersection of the center line and the contour box as the tip position.
Due to the development of advanced segmentation algorithms, we believe that real-time medical image segmentation, which is another popular research topic with advanced development, is already capable of handling common surgical scenarios. Therefore, this segmentation model is only used to extract prior image components for algorithm demonstration.

\subsubsection{Results}
The procedures of approaching two types of targets are presented in Fig.~\ref{fig:exp_real}, (a)-(d) for the retinal target and (e)-(h) for the floating target.
Here, the target (red) and its shadow (dark red) are overlapped in the retinal case, and are separated in the floating case.
As can be seen from the figure, the needle (green) is controlled together with its shadow (gray) during the approach, with both needle-target and their shadows almost overlapped.
During the "optimize" step, the needle is controlled with axial insertion to move the shadow tip (gray) towards the target shadow along the ideal shadow trajectory.
Therefore, the model test qualitatively proves the shadow utilization proposed in this paper to automatically control the needle towards the target.

Meanwhile, a sporadic shifting of the needle tip within a small area is observed, which results from unstable mask generation during image segmentation.
Therefore, the clinical deployment of the proposed shadow utilization should use advanced segmentation models to prevent mistaken needle manipulations caused by instability or failure of segmentation.
\begin{figure*}[ht]
\centering
    \begin{subfigure}{0.24\textwidth}
        \includegraphics[width=0.95\textwidth]{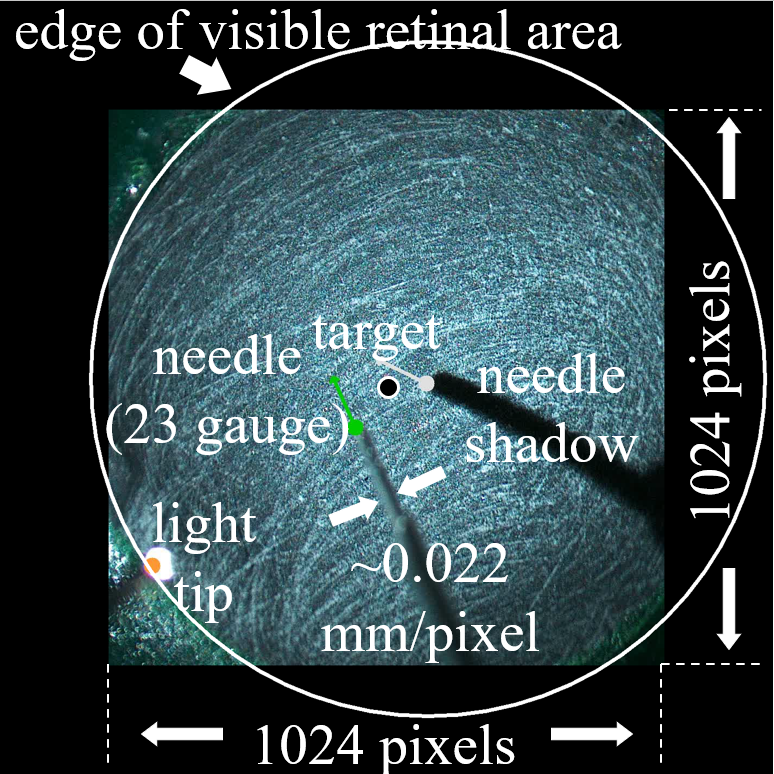}
        \caption{(retinal) init}
    \end{subfigure}
    \begin{subfigure}{0.24\textwidth}
        \includegraphics[width=0.95\textwidth]{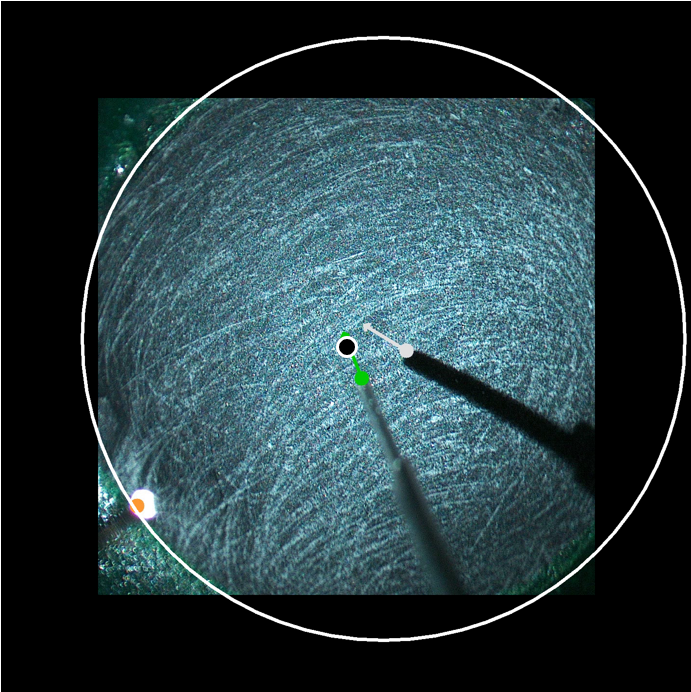}
        \caption{(retinal) aligning angle}
    \end{subfigure}
    \begin{subfigure}{0.24\textwidth}
        \includegraphics[width=0.95\textwidth]{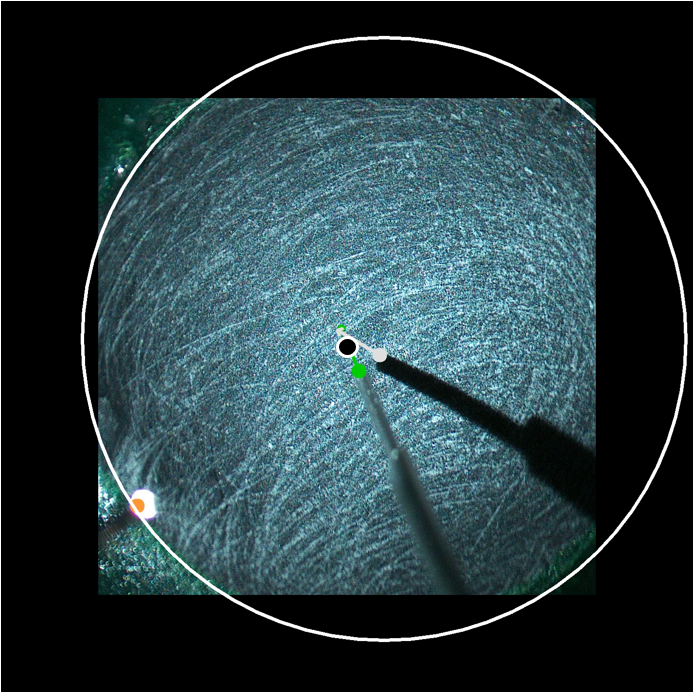}
        \caption{(retinal) insertion}
    \end{subfigure}
    \begin{subfigure}{0.24\textwidth}
        \includegraphics[width=0.95\textwidth]{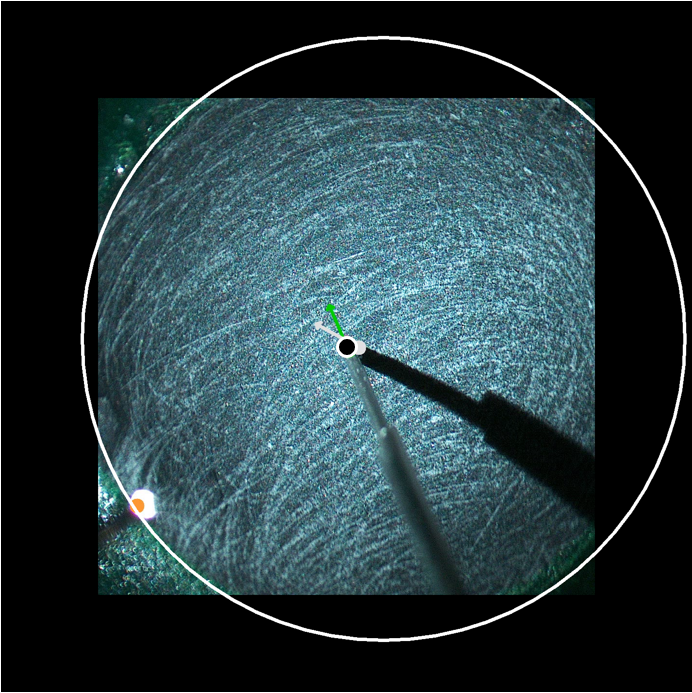}
        \caption{(retinal) aligning shadow}
    \end{subfigure}
    \begin{subfigure}{0.24\textwidth}
        \includegraphics[width=0.95\textwidth]{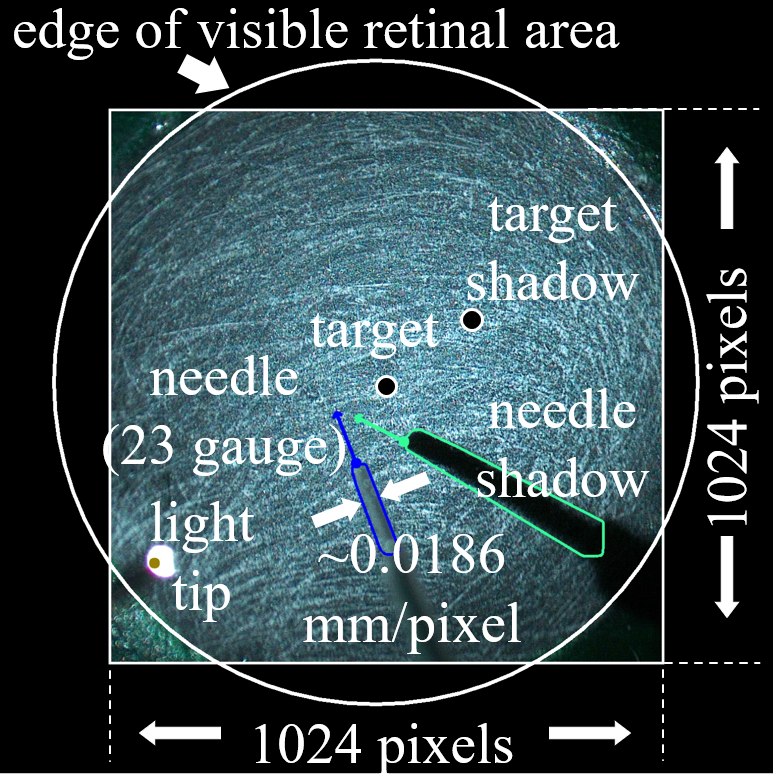}
        \caption{(floating) init}
    \end{subfigure}
    \begin{subfigure}{0.24\textwidth}
        \includegraphics[width=0.95\textwidth]{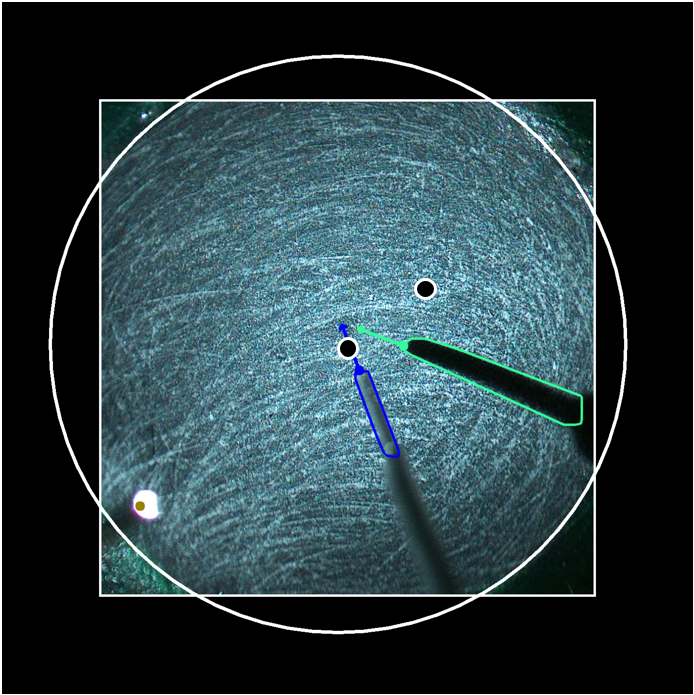}
        \caption{(floating) aligning angle}
    \end{subfigure}
    \begin{subfigure}{0.24\textwidth}
        \includegraphics[width=0.95\textwidth]{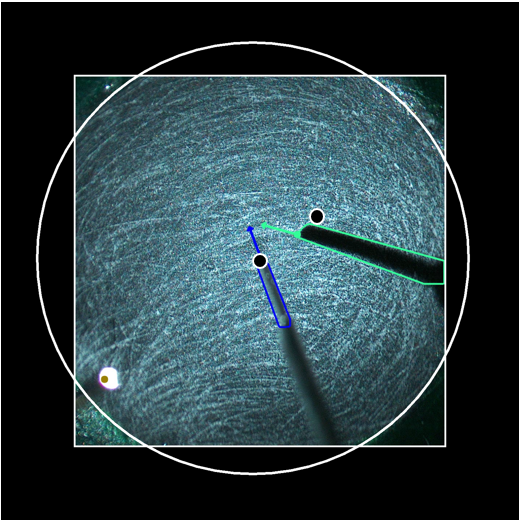}
        \caption{(floating) insertion}
    \end{subfigure}
    \begin{subfigure}{0.24\textwidth}
        \includegraphics[width=0.95\textwidth]{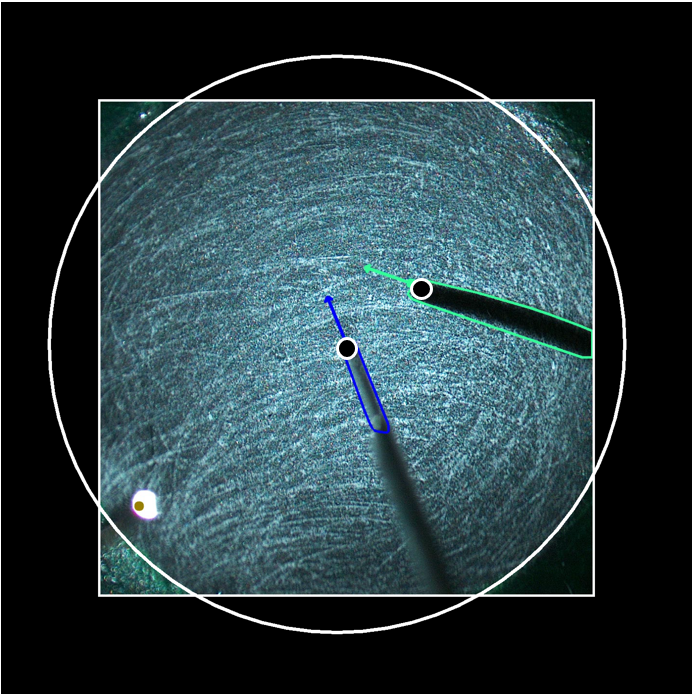}
        \caption{(floating) aligning shadow}
    \end{subfigure}
    \caption{The procedure of approaching virtual retinal ((a)-(d)) and floating ((e)-(h)) targets using our proposed alignment strategies.}\label{fig:exp_real}
\end{figure*}

\section{Conclusion and Discussion}
In ophthalmic surgeries, surgeons manipulate light probe and instruments to cast shadows within the visible area and implicitly optimize their control of instruments.
The proposed method in this paper combines this shadow-based intraocular perception with the shadow-based collision avoidance~\cite{9695979} to enable robot-assisted trajectory optimization during the needle navigation towards both floating and retinal targets.
However, the proposed method is significantly affected by the unstable segmentation quality and the coarse-grained distance estimation due to retinal transparency, bringing more challenges to guaranteeing the intraocular safety.
Moreover, the light probe's deterministic role in shadows' positioning and quality control also brings more research focus to the instrument-light coordination.

%\clearpage
%\newpage
\bibliographystyle{IEEEtran}
\bibliography{ref}

\end{document}